\documentclass{article}

\usepackage{microtype}
\usepackage{graphicx}
\usepackage{subcaption}
\usepackage{booktabs} 
\usepackage{diagbox}
\usepackage{multirow}
\usepackage{hyperref}
\usepackage{enumitem}

\usepackage[preprint]{icml2026}

\usepackage{amsmath}
\usepackage{amssymb}
\usepackage{mathtools}
\usepackage{amsthm}

\usepackage[capitalize,noabbrev]{cleveref}
\usepackage{notation}

\theoremstyle{plain}
\newtheorem{theorem}{Theorem}[section]
\newtheorem{proposition}[theorem]{Proposition}

\newtheorem{corollary}[theorem]{Corollary}
\theoremstyle{definition}
\newtheorem{definition}[theorem]{Definition}
\newtheorem{assumption}[theorem]{Assumption}
\theoremstyle{remark}

\usepackage[textsize=tiny]{todonotes}

\icmltitlerunning{Your GFlowNet Secretly Learns an Optimal Transport Plan}

\begin{document}

\twocolumn[
  \icmltitle{Your GFlowNet Secretly Learns an Optimal Transport Plan}



  \icmlsetsymbol{equal}{*}

  \begin{icmlauthorlist}
    \icmlauthor{Ian Maksimov}{yyy}
    \icmlauthor{Nikita Morozov}{yyy}
    \icmlauthor{Denis Belomestny}{yyy,comp}
    \icmlauthor{Sergey Samsonov}{yyy}
  \end{icmlauthorlist}

  \icmlaffiliation{yyy}{HSE University}
  \icmlaffiliation{comp}{Duisburg-Essen University}

  \icmlcorrespondingauthor{Ian Maksimov}{yavmaksimov@edu.hse.ru}

  \icmlkeywords{Machine Learning, ICML}

  \vskip 0.3in
]



\printAffiliationsAndNotice{}  


\begin{abstract}
Generative Flow Networks (GFlowNets) are a framework for sampling structured objects via stochastic trajectories in a directed graph. In this work, we establish a theoretical connection between non-acyclic GFlowNets and optimal transport (OT). We show that fixing the initial flow distribution in a minimum-flow GFlowNet reduces its objective to a Kantorovich OT problem with graph-induced shortest path costs. At the optimum, the learned GFlowNet policy therefore encodes an optimal transport plan from the source distribution to the target distribution: we show that sampling trajectories from the minimum-flow GFlowNet recovers the corresponding optimal coupling. Our formulation enables applying the GFlowNet learning framework to OT problems on large graphs via edge flows and neural parameterization. Experiments confirm agreement with exact OT solvers and demonstrate that GFlowNets can learn high-quality transport plans.
\end{abstract}

\section{Introduction}

\label{sec:intro}
Generative Flow Networks (GFlowNets, \citealp{bengio2021flow}) are a class of methods for sampling structured discrete objects from probability distributions given by unnormalized probability mass function. A GFlowNet constructs objects sequentially through stochastic transitions induced by a forward policy. GFlowNets have been successfully used in a range of applications, including molecule generation \cite{bengio2021flow, shen2024tacogfn, koziarski2024rgfn, cretu2025synflownet}, biological sequence design \cite{jain2022biological, kim2024improvedoff}, combinatorial optimization~\cite{zhang2023solving, zhang2023robust, kim2025ant}, and the fine-tuning of large language models and diffusion models~\cite{hu2023amortizing, venkatraman2024amortizing, uehara2024understanding, zhang2024improving, kwon2024gdpo, lee2025learning, zhu2025flowrl}. The theoretical foundations of GFlowNets were established in \cite{bengio2023gflownet,lahlou2023theory}.

GFlowNet define a generation process that follows a sequence of stochastic transitions of discrete nature, inducing a graph structure $\cG$. The generation starts at a special initial state $s_0$, and then moves along the edges of the graph until reaching a special sink state $s_f$. The last state $x$ from which the process transitioned into $s_f$ is then considered the final sample. While GFlowNets were initially developed to operate in acyclic graph environments~\cite{bengio2021flow}, further works generalized them to non-acyclic graphs~\cite{brunswic2024theory, pmlr-v267-morozov25a}. In this setting, the expected length of a trajectory sampled in the graph becomes a key quantity controlling sampling efficiency. To effectively minimize it in practice, ~\citet{brunswic2024theory} and \citet{pmlr-v267-morozov25a} proposed flow regularization techniques. Moreover, non-acyclic GFlowNets have been recently applied to tasks that lie outside the domain of sampling, e.g., imitation learning~\cite{brunswic25ergodic}, multi-agent learning~\cite{brunswic2025theory} and shortest path problems \cite{morozov2026learningshortestpathsgenerative}, further demonstrating the generality of the framework.

Optimal transport, on the other hand, provides a principled way to compare probability distributions by taking into account the geometry and structure of the underlying space \citep{villani2008optimal, peyre2019computational}. The central object is a transport plan: a coupling that specifies how much mass should be moved from each source point to each target point while minimizing a prescribed transportation cost. This formulation is especially attractive in machine learning because many data objects - images, documents, point clouds, empirical samples, and learned representations can naturally be viewed as discrete distributions. Unlike purely pointwise divergences, optimal transport distances remain meaningful and can exploit semantic or geometric structure through the ground cost. This has led to applications in image retrieval via Earth Mover’s Distance \citep{rubner2000earth}, natural language processing via Word Mover’s Distance \citep{kusner2015word}, domain adaptation \citep{courty2017optimal}, and generative modeling via Wasserstein objectives \citep{arjovsky2017wasserstein}, as well as representation learning and barycenter-based methods \citep{tolstikhin2018wasserstein,agueh2011barycenters}.

In this paper, we uncover a connection between the minimum-flow formulation of non-acyclic GFlowNets and discrete optimal transport. The main idea of this connection is based on two facts: 1) if an additional constraint is added to probabilities of transitions from the initial state $s_0$, GFlowNet effectively induces a transport plan between two probability distributions; 2) minimizing the total flow is equivalent to minimizing the graph distance transport cost of the aforementioned transport plan. This equivalence provides a novel view on the GFlowNet learning problem, as well as allows for the plethora of existing GFlowNet learning algorithms to be applied to OT problems.

\begin{figure}
    \centering
    \includegraphics[width=0.9\linewidth]{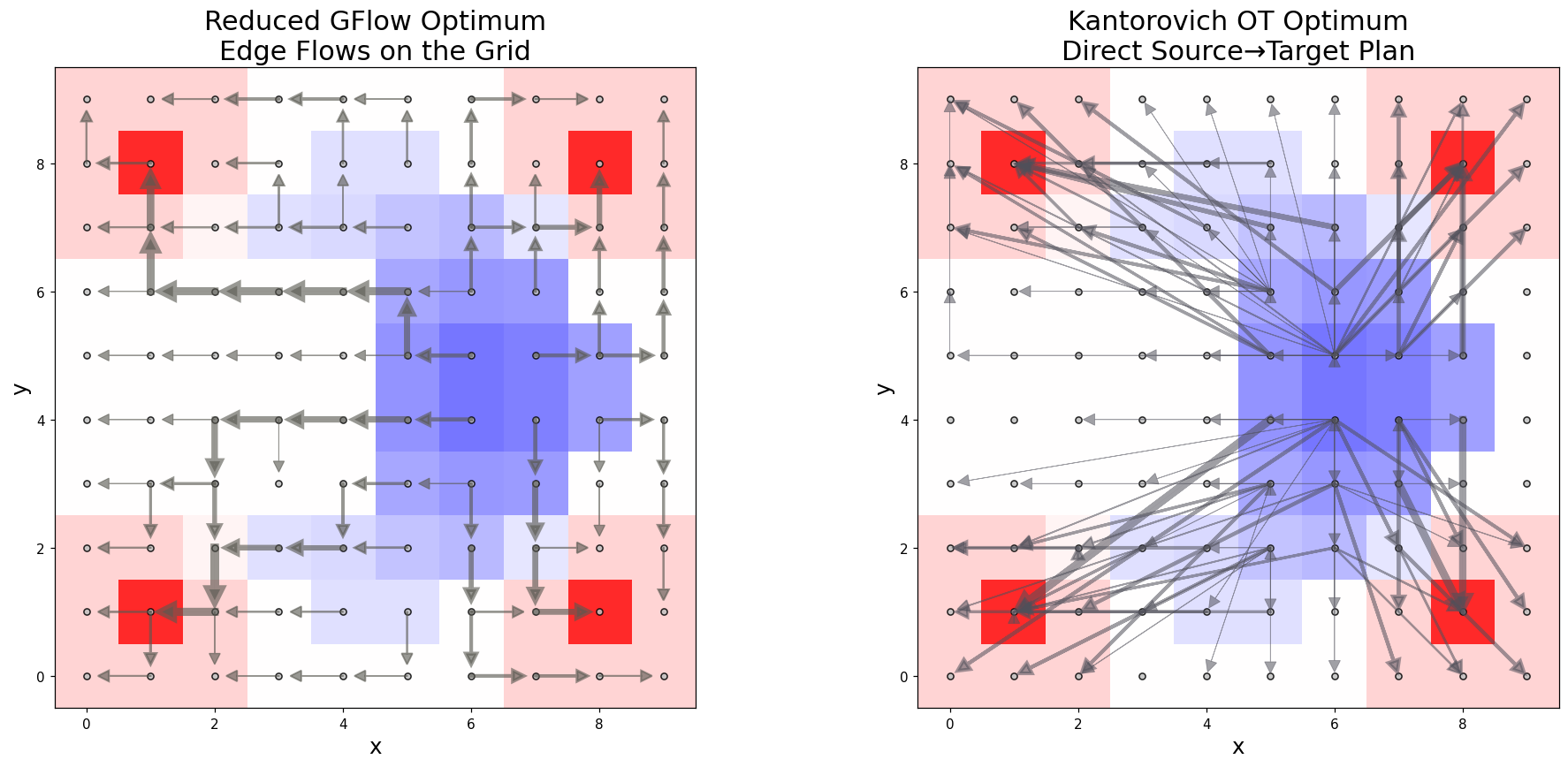}
    \caption{Visualization of the solution to the GFlowNet LP problem and Kantorovich OT optimal plan on a hypergrid environment. OT induces optimal coupling that connects source and target states and how much mass is transfered between them. GFlowNet on the other hand builds a policy that samples target states starting from source states, thus inducing a coupling with respect to the shortest path metric. Both formulations achieve the same transport cost of $4.351$. \vspace{-0.2cm}}
    \label{fig:exact_solutions}
\end{figure}

The main contributions of the paper can be summarized as follows:

\begin{enumerate}
    \item We show that the minimum-flow non-acyclic GFlowNet learning problem~\cite{pmlr-v267-morozov25a} can be equivalently formulated as a linear programming problem. We then theoretically demonstrate that fixing the initial edge-flow distribution in it effectively turns the minimum-flow objective into a Kantorovich optimal transport problem where the cost is defined as a distance in the graph, which is also equivalent to the discrete graph-based formulation of the Beckmann problem~\cite{beckmann1952continuous, essid2018quadratically}. The optimal GFlowNet forward policy in this formulation samples optimal paths between the initial and the terminal distributions, thus inducing an optimal coupling.  As a particular case, this construction recovers the shortest path theory of~\citet{morozov2026learningshortestpathsgenerative}.
    \item We show that the GFlowNet learning framework provides an effective way to approximate solutions to graph optimal transport problems. Based on a neural parameterization, GFlowNet objectives allow to learn a policy that aims to recover an optimal coupling. We experimentally demonstrate the ability of the framework to recover exact OT solutions in cases when they are tractable, as well as its scalability, exhibited by the ability to efficiently approximate solutions with the growth of the combinatorial space where exact solutions become intractable. We thus expand the generality of the GFlowNet framework, showing that it extends into the computational optimal transport domain.
\end{enumerate}

\vspace{-0.2cm}

\section{Background}

\vspace{-0.1cm}
\subsection{Non-Acyclic GFlowNets}\label{sec:gflow_background}

\citet{brunswic2024theory} and \citet{pmlr-v267-morozov25a} extend the theoretical construction of GFlowNets \citep{bengio2023gflownet} to non-acyclic environments. We use theory and notation of \citet{pmlr-v267-morozov25a} which is summarized below as a starting point.

The sequential sampling process is described by a directed graph \(\cG = (\mathcal{S}, \mathcal{E})\), where \(\mathcal{S}\) is a finite state space and \(\mathcal{E} \subseteq \mathcal{S} \times \mathcal{S}\) is a finite set of edges (or transitions). We will write \(\vout(s)\) for the set of children of \(s\), and \(\vin(s)\) for the set of parents of \(s\).

Let $\cT$ be a set of all finite trajectories $\tau = \left(s_0 \to s_1 \to \ldots \to s_{n_{\tau}} \to s_f\right)$ from $s_0$ to $s_f$, where we use $n_{\tau}$ to denote the length of the trajectory $\tau$. We use a standard convention that $s_{n_{\tau} + 1} = s_f$. We will say that $\tau$ stops at $s$ if its last transition is $s \to s_f$. The edges of the form $(s \to s_f)$ are called terminating transitions, and the states $s$ that have an outgoing edge into $s_f$ are called \textit{terminal states}. The set of terminal states is denoted by $\cX$, and the probability distribution of interest $\cR(x) / \cZ$ is defined on it, where $\cR(x) > 0$ is called GFlowNet reward and $\cZ = \sum_{x \in X} \cR(x)$ is an unknown normalizing constant. 

GFlowNet is essentially a pair of bidirectional policies, where a \textit{forward policy} $\PF(s' \mid s)$ is a distribution over children of each state, and a \textit{backward policy} $\PB(s \mid s')$ is a distribution over parents of each state. GFlowNet is trained in a way to find such pair of policies that the distributions over trajectories they induce coincide:
\begin{equation}
\label{eq:tb}
\textstyle 
     \cP(\tau) = \prod_{t=0}^{n_\tau} \PF \left(s_{t+1} \mid s_{t}\right) = \prod_{t=0}^{n_\tau} \PB \left(s_{t} \mid s_{t+1}\right)\,.
\end{equation}
The main goal for sampling is to have such $\cP$ that for any $x \in X$, probability that $\tau \sim \cP$ terminates in $x$ coincides with $\cR(x) / \cZ$. This property is called the \textit{reward matching condition}, and w.r.t $\PB$ it can be simply written as 
\begin{equation}
\label{eq:rm}
\PB(x \mid s_f) = \cR(x) / \cZ \;\; \forall x \in X.
\end{equation}
If both \eqref{eq:tb} and \eqref{eq:rm} are satisfied, $\PF$ can be used to sample terminal states from the reward distribution (see \citet{bengio2023gflownet} and \citet{pmlr-v267-morozov25a} for more details).


\citet{pmlr-v267-morozov25a} defines state and edge flows through the expected number of visits of a certain state
\begin{equation}
\begin{aligned}
    \label{eq:flow}
    \cF(s) &=  \cZ \cdot \;\E_{\tau \sim \cP}\left[ \sum_{t = 0}^{n_{\tau}+1} \ind\{s_t = s\}\right],  \quad \\
    \cF(s \to s') &=  \cZ \cdot \;\E_{\tau \sim \cP}\left[ \sum_{t = 0}^{n_{\tau}} \ind\{s_t = s, s_{t+1} =s'\}\right]
\end{aligned}
\end{equation}
It is clear from the definition that the flow functions satisfy the following conditions (Proposition 3.6 of \citet{pmlr-v267-morozov25a}):
\begin{equation}
    \label{eq:fm_conditions}
    \sum_{s \to s'} \cF(s \to s') = \sum_{s'' \to s} \cF(s'' \to s)
\end{equation} and $\cF(s_0) = \cF(s_f) = \cZ.$

The main problem arising when extending standard GFlowNet framework to non-acyclic is that the $n_\tau$ is unbounded, since the policy can go in cycles. $\E[n_\tau]$ exactly represents the mean trajectory length needed to obtain a sample from target distribution following actions sampled from forward policy, and it also signifies the practical efficiency of the sampler. Thus, it is natural to formulate the problem of training a non-acyclic GFlowNet that minimizes $\E[n_\tau]$. Using the above definition, one check that the expected trajectory length coincides with the normalized total flow $\E[n_\tau] = \tfrac{1}{\cZ} \sum_{s \in \cS \setminus \{s_0, s_f\}} \cF(s)$, thus \citet{pmlr-v267-morozov25a} gives a theoretical ground for learning a model with the smallest $\E[n_\tau]$ which is equivalent to learning a model with the smallest total flow.


\citet{pmlr-v267-morozov25a, morozov2026learningshortestpathsgenerative} formulates the non-acyclic GFlowNet problem as the following constraint optimization problem. Define the set of internal states $I := \cS \setminus \{s_0,s_f\}.$, then:
\begin{equation}
\label{eq:gflownet-min-flow}
\begin{aligned}
&\min_{\cF,\PF,\PB}\quad
 \sum_{s\in I} \cF(s) \\
\text{s.t.}\quad
& \cF(s)\PF(s'\mid s)=\cF(s')\PB(s\mid s') \,, \forall (s,s')\in \cE, \\
& \cF(s_f)\PB(x\mid s_f)=\cR(x),
\quad x\in X.
\end{aligned}
\end{equation}
The constraint $\cF(s)\PF(s'\mid s)=\cF(s')\PB(s\mid s')$ is referred to as the detailed balance condition. As shown in~\cite{bengio2023gflownet,malkin2022trajectory, pmlr-v267-morozov25a}, satisfying detailed balance for every adjacent pair of states is equivalent to the trajectory balance condition~\eqref{eq:tb}.

This optimization problem naturally encodes all necessary constraints needed for training GFlowNet. Equality constraints satisfy Equations \ref{eq:tb} \& \ref{eq:rm}, while the objective minimizes total flow, which is equivalent to minimizing $\E[n_\tau]$ as per \citet{pmlr-v267-morozov25a}.

The simplest approach to this problem proposed in \citet{brunswic2024theory} and further explored in \citet{pmlr-v267-morozov25a} is to add $\lambda \cF_\theta(s)$ to the utilized GFlowNet loss e.g detailed balance \cite{bengio2023gflownet} or variants of trajectory balance loss ~\cite{madan2023learning, malkin2022trajectory, morozov2026learningshortestpathsgenerative}, where $\lambda$ is a regularization coefficient.

\subsection{Kantorovich OT via Linear Programming}
While GFlowNet framework aims to sample from some target distribution $\cR(x)/\cZ$, optimal transport problem on the other hand seeks an optimal coupling of two given source and target distributions, that moves probability mass from the source distribution to the target with respect to the defined cost function \cite{villani2008optimal}.
\begin{definition}
    Let $\mu$ and $\nu$ be two probability measures on a finite discrete spaces $\cX$ and $\cY$. A coupling of $\mu$ and $\nu$ is a joint probability distribution $\Pi$ on space $\cX \times\cY$ such that:
    \begin{equation*}
        \sum_{i=1}^{n}\Pi_{i,j} = \nu_j, \quad
            \sum_{j=1}^{m} \Pi_{i, j} = \mu_i
    \end{equation*}
\end{definition}

Suppose we have two finite set of states $\{x_1, \dots x_n\}$ and $\{y_1 \dots y_m\}$ with probability distributions $\mu, \nu$. Then Kantorovich OT problem is formulated as follows:
\begin{equation}
\label{eq:ot-kantorovich}
\begin{aligned}
\textstyle
    \min\limits_{\Pi \geq 0} \quad
    & \textstyle  \sum_{i,j=1}^{n,m}d(x_i, y_j)\Pi_{i,j} \\
    \textstyle \text{s.t.}\quad
    & \textstyle \sum_{i}\Pi_{i, j} = \nu_j, \quad \sum_{j}\Pi_{i, j} = \mu_i\,.
\end{aligned}    
\end{equation}
where $d(x_i, y_j)$ is some cost of moving mass form $x_i$ to $y_j$.
Usually this linear programming problem is solved via classical methods e.g. simplex-method \cite{dantzig1951maximization}. 

\section{GFlowNets and Optimal Transport}
\label{mainsec}

Consider a non-acyclic GFlowNet environment graph $\cG = (\cS, \cE)$. We impose the following assumptions on the structure of the graph $\cG$, similarly to \citet{essid2018quadratically}:
\begin{assumption}
\label{graph_assumption}
     \begin{itemize}[noitemsep,topsep=0pt,leftmargin=2em]
         \item There is a special \textit{initial state} $s_0$ with no incoming edges and a special \textit{sink state} $s_f$ with no outgoing edges;
         \item There are sets of states $U$ and $X$ s.t. $s_0$ has outgoing edges only to $U$ and $s_f$ has incoming edges only from $X$;
         \item There exists a finite-length path between any $u \in U$ and $x \in X$; 
         \item We are given two distributions $\cL(u)$ and $\cR(x)$ on $U$ and $X$, respectively, that is,
         \[
         \sum_{u \in U} \cL(u) = \sum_{x \in X} \cR(x) = 1\,.
         \]
         We adopt the convention that both $\cL$ and $\cR$ are extended by zero outside their supports and thus defined on all internal states $I$.
     \end{itemize}
\end{assumption} 
Next, define a cost function $d(u,x)$ for pairs of $u \in U$ and $x \in X$ as 
\begin{equation}
\label{eq:d_u_x_def}
d(u,x) = |\tau_{u,x}|\,,
\end{equation}
where $\tau_{u,x}$ is a shortest path from $u$ to $x$ and $|\tau_{u,x}|$ denotes its length. Thus, one can consider a Kantorovich OT problem \eqref{eq:ot-kantorovich} between distributions $\cL$ and $\cR$ with the cost function $d(u,x)$:
\begin{equation}
\label{eq:ot-graph-problem}
\begin{aligned}
\textstyle \min\limits_{\Pi \geq 0} \quad & \textstyle  \sum_{i,j=1}^{n,m}d(u_i, x_j)\Pi_{i,j} \\ \textstyle \text{s.t.}\quad & \textstyle \sum_{i}\Pi_{i, j} = \cR(x_j)\,, \quad
\sum_{j}\Pi_{i, j} = \cL(u_i)\,.
\end{aligned}    
\end{equation}
Where $\Pi \in \mathbb{R}_{\geq 0}^{n \times m}$ is a matrix which defines some coupling between the distributions $\cL(u)$ on $U$ and $\cR(x)$ on $X$. Here $|U| = n$, and $|X| = m$. It is also worth noting that this graph optimal transport with shortest-path ground costs has a classical min-cost-flow formulation \cite{essid2018quadratically}. We next show that the minimum-flow objective arising in non-acyclic GFlowNets, once the initial edge-flow distribution is fixed, reduces exactly to this graph OT problem. 



\subsection{Minimum-flow LP formulation}

We start with the GFlowNet minimum-flow problem \eqref{eq:gflownet-min-flow} that was introduced in~\citet{pmlr-v267-morozov25a}. The detailed balance constraints in \eqref{eq:gflownet-min-flow} are bilinear in the decision variables, since
\begin{equation}
\label{eq:db_lp}
\cF(s\to s') = \cF(s)\PF(s'\mid s)=\cF(s')\PB(s\mid s')\,.
\end{equation}
The problem becomes linear after policies \(\PF\) and \(\PB\) are eliminated and replaced by independent state and edge-flow variables. We will now treat every state-flow $\cF(s)$, $s\in \cS$, and edge-flow $\cF(s\to s')$, $s\to s'\in \cE$, as an independent variable. 
\par 
In this case, the detailed balance condition \eqref{eq:db_lp} from the optimization problem \eqref{eq:gflownet-min-flow} can be substituted with the flow matching constraint \eqref{eq:fm_conditions}, allowing for an equivalent reformulation of \eqref{eq:gflownet-min-flow} as a linear program. We provide the proof of this equivalence in Appendix \ref{app:correctens_lp}.



\par 
We fix the distribution $\cL(u)$ defined on $U$ which we will call Leward (left reward). Then we fix the first-step distribution induced by the underlying GFlowNet similarly to the reward matching condition \eqref{eq:rm}:
\[
\cF(s_0 \to u) = \cL(u)\,.
\]
After adding this constraint, the extended GFlowNet linear programming problem writes as
\begin{equation}
\label{eq:extended-primal-lp}
\begin{aligned}
&\min_{\cF(s),\cF(s \to s')}\quad
\sum_{s\in I} \cF(s) \\
\text{s.t.}\quad
& \sum_{v\in\vout(s)} \cF(s \to v)=\cF(s),
\quad s\in \cS\setminus\{s_f\}, \\
& \sum_{u\in\vin(s)} \cF(u \to s)=\cF(s),
 \quad s\in \cS\setminus\{s_0\}, \\
& \cF(x \to s_f)=\cR(x),
\quad x\in X, \\
& \cF(s_0 \to u)=\cL(u),
\quad u\in U, \\
& \cF(s \to s')\ge 0,
 \quad (s,s')\in \cE.
\end{aligned}
\end{equation}

From GFlowNets theoretical construction, it is unclear whether such problems with additional equality constraints has a feasible solution. However, we further show that a solution always exists (Theorem \ref{th:otequivalence})

Extended primal problem can be reduced by omitting the flow variable \(\cF(s)\). Define
\[
E^\circ := \{s\to s'\in \cE : s\neq s_0,\ s'\neq s_f\},
\]
The reduced primal formulation is: 
\begin{equation}
\label{eq:reduced-primal-lp}
\begin{aligned}
&\min_{\cF(s \to s')\ge 0}\quad
 \sum_{s\to s'\in E^\circ} \cF(s \to s') \\
\text{s.t.}\quad
& \sum_{v:\,s\to v\in E^\circ} \cF(s \to v)\\
       &- \sum_{u:\,u\to s\in E^\circ} \cF(u \to s)=\cL(s) - \cR(s),
\qquad s\in I.
\end{aligned}
\end{equation}
Note that when reducing the initial problem, we have the relation
\begin{align*}
\sum_{s \in I} \cF(s) &= \sum_{s_0 \to u} \cF(s_0 \to u) + \sum_{s \to s' \in E^\circ} \cF(s \to s')\\ &= 1 + \sum_{s \to s' \in E^\circ} \cF(s \to s')\,.    
\end{align*}

We will omit this constant, as it does not change minimum of the problem.

Note that if one sums up flow-matching constraint for all internal states, one would have: 
\[
\sum_{x \in X} \cR(x) = \sum_{u \in U} \cL(u).
\]
This equality imposes a mass-balance constraint: the two functions must either have equal total mass or be normalized as probability distributions. We adopt the latter formulation, as it eliminates the need to handle an unknown normalizing constant. This is advantageous both in the theoretical analysis and in practical implementations, where the constant would otherwise need to be learned or estimated. Accordingly, we include this condition in Assumption~\ref{graph_assumption}.


\subsection{GFlowNet-OT equivalence}

We are now ready to establish the equivalence between the Kantorovich OT problem \eqref{eq:ot-graph-problem} and the GFlowNet problem \eqref{eq:reduced-primal-lp}. Denote by $\text{GFlow}^\star$ the optimal value of the functional \eqref{eq:reduced-primal-lp} and by $\text{OT}^\star$ the optimal value of the functional \eqref{eq:ot-graph-problem} in the Kantorovich problem. Then the following equivalence holds:

\begin{theorem}
\label{th:otequivalence}
Under Assumption \ref{graph_assumption}, the reduced primal problem \eqref{eq:reduced-primal-lp} is equivalent to the OT Kantorovich problem \eqref{eq:ot-graph-problem} between $\cL$ on $U$ and $\cR$ on $X$, that is,  
\begin{equation}
\label{eq:ot-graph-problem-value}
\text{GFlow}^{\star} = \text{OT}^{\star}\,.
\end{equation}
Moreover, if $\cP^{\star}$ is the trajectory distribution induced by the solution of $\text{GFlow}^\star$, it induces an optimal coupling $\Pi^{\star}_{u,x} := \sum_{\tau:u\leadsto x} \cP^{\star}(\tau)\,,  u \in U\,, x \in X$ in \eqref{eq:ot-graph-problem}.
\end{theorem}

\begin{proof}
    1. Suppose that we are given an optimal coupling $\Pi_{u,x}^\star$ which is the solution of the minimization problem \eqref{eq:ot-graph-problem}.
    Existence of $\Pi_{u,x}^\star$ is guaranteed by Theorem 4.1 from \citep[p.~43]{villani2008optimal}. We construct a feasible flow in a sense of \eqref{eq:reduced-primal-lp} as follows. For every $u \in U$ and $x \in X$, we choose an arbitrary shortest path $\tau_{u,x}$ from $u$ to $x$. Then we set
    \[\cF(s \to s') = \sum_{u\in U}\sum_{x\in X} \Pi_{u,x}^\star \cdot \mathbb{I}[s \to s' \in \tau_{u,x}]\]
    Let us check the feasibility of this flow. Define for $s \in I$, $u \in U$, and $x \in X$, the auxiliary function
    \[
    \Delta^{u,x}_{s} = \sum_{v:s\to v \in E} \mathbb{I}[s\to v \in \tau_{u,x}] - \sum_{v: v\to s\in E} \mathbb{I}[v \to s \in \tau_{u,x}]\,.
    \]

    Since $\tau_{u,x}$ is a directed path, it has only outgoing edge at its start, only incoming edge at its end, and equal incoming/outgoing counts at every intermediate node. So we can equivalently write 
    \[
    \Delta^{u,x}_{s} = \mathbb{I}[s = u] - \mathbb{I}[s = x]
    \]

Let $\Delta\cF(s)$ denote the left-hand side of the equality constraint in \eqref{eq:reduced-primal-lp}. Using
$\Delta_s^{u,x}=\ind[s=u]-\ind[s=x]$ and the marginal constraints
of $\Pi^\star$, we have
{
\begin{equation}
\begin{aligned}
\Delta\cF(s)
&= \sum_{u\in U,\,x\in X}\!\Pi_{u,x}^\star
   \bigl(\ind[s=u]-\ind[s=x]\bigr) \\
&= \sum_{u\in U}\ind[s=u]\cL(u)
   - \sum_{x\in X}\ind[s=x]\cR(x)\\
 &= \cL(s)-\cR(s).
\end{aligned}
\end{equation}}
The above equations implies the equality constraint from \eqref{eq:reduced-primal-lp}, therefore we have checked that the flow we defined is feasible, thus the solution set is non-empty. Now let us compute the value of \eqref{eq:reduced-primal-lp}:
\begin{equation}
    \begin{aligned}
            \sum_{s \to s' \in E^\circ}&\cF(s \to s') = \sum_{\substack{s \to s' \in E^\circ\\ u\in U,\; x\in X}} \Pi_{u,x}^\star \mathbb{I}[s \to s' \in \tau_{u,x}] \\
            &= \sum_{u\in U, x\in X}\Pi_{u,x}^\star  \sum_{s\to s'\in E^\circ}\mathbb{I}[s \to s'  \in \tau_{u,x}] \\
            &= \sum_{u\in U, x \in X} \Pi_{u,x}^\star |\tau_{u,x}|\,,
    \end{aligned}
\end{equation}

where $|\tau_{u,x}|$ is the length of the path from $u$ to $x$. Since $\tau_{u,x}$ is a shortest path, $|\tau_{u,x}| = d(u,x)$, and 
\begin{equation}
   \sum_{s \to s' \in E^\circ} \mathcal{F}(s \to s')= \sum_{u,x} d(u,x)\Pi_{u,x}^{\star}.
\end{equation}
The flow we constructed is feasible, but might not be optimal. Hence, we obtain the inequality
    \begin{equation}
    \label{eq:OT_star_geq} 
    \text{GFlow}^\star \leq \text{OT}^\star\,.
    \end{equation}
    2. Note that the optimal solution exists under Assumptions \ref{graph_assumption} and by the fact that our problem is bounded below. Take this optimal solution to the problem \eqref{eq:reduced-primal-lp}, and denote it by $\cF^{\star}(s \to s')$. This flow follows equivalence establised in \eqref{app:correctens_lp}, this means that we now may introduce forward policy: 
    \[
    \PF^\star(s' | s) = \cF^{\star}(s \to s') / \sum_{s\to s''} \cF^{\star}(s \to s'')
    \]
    and the respective probability over trajectories $\tau=(v_0 = s_0, v_1 \dots v_{n_\tau -1},v_{n_\tau} = s_f)$: 
    \[
    \cP^\star(\tau) = \prod_{i = 0}^{n_\tau} \PF^{\star}(v_{i+1} | v_i)\,.
    \]
    Then the definition \eqref{eq:flow} from \citet{pmlr-v267-morozov25a} implies that 
    \[
    \cF^{\star}(s \to s')= \cZ \sum_{\tau \in \cT} \cP^\star(\tau) \sum_{t = 0}^{n_{\tau}} \ind[s_t = s, s_{t+1} =s']\,.
    \]
    Assumption \ref{graph_assumption} implies that $\cZ = 1$,  so we omit it in further derivations. Now define the joint distribution $\Pi$ on $U \times X$ by 
    \[
    \Pi_{u,x} := \sum_{\tau:u\leadsto x} \cP^{\star}(\tau)\,, \quad u \in U\,, \quad x \in X\,.
    \]
    It is clear that 
    \begin{align*}
        \sum_{x \in X}\Pi_{u,x} &= \sum_{x\to s_f}\sum_{\tau:u\leadsto x} \cP^{\star}(\tau) \\
        &= \sum_{\tau \ni s_0 \to u} \cP^{\star}(\tau)= \PF^{\star}(u | s_0)= \cL(u)\,,
    \end{align*}
    and, similarly, 
    \[
    \sum_{u \in U}\Pi_{u,x} = \cR(x)\,.
    \]
    Thus, the joint distribution $\Pi$ defined above is a coupling of $\cR(x)$ and $\cL(u)$. Let us write down the value of primal problem using our flow:
    \begin{equation}
        \begin{aligned}
            \text{GFlow}^\star &= \sum_{s \to s' \in E^\circ} \cF^{\star}(s \to s')\\ &=\sum_{s \to s'\in E^\circ} \sum_{\tau} \cP^{\star}(\tau)  \sum_{t = 0}^{n_{\tau}} \ind[s_t = s, s_{t+1} =s']\\
            &=  \sum_{\tau} \cP^{\star}(\tau) \sum_{s \to s'\in E^\circ} \sum_{t = 0}^{n_{\tau}} \ind[s_t = s, s_{t+1} =s']\\ &= \sum_{\tau} |\tau|\cP^{\star}(\tau)\,.
        \end{aligned}        
    \end{equation}
    Recall that $|\tau_{u,x}| \geq  d(u,x)$, since $d(u, x)$ is the length of the shortest path. Thus:
    \begin{equation}
        \begin{aligned}
            \sum_{\tau} |\tau| \cP^{\star}(\tau)
            &=\sum_{s_0\to u}\sum_{x \to s_f}\sum_{\tau: u\leadsto x} |\tau|\cP^{\star}(\tau)\\ &\ge \sum_{s_0\to u}\sum_{x \to s_f}\sum_{\tau: u\leadsto x} d(u, x)\cP^{\star}(\tau)\\
            &= \sum_{s_0\to u}\sum_{x \to s_f} d(u, x)\sum_{\tau: u\leadsto x}\cP^{\star}(\tau) \\ &= \sum_{u \in U}\sum_{x \in X} d(u, x)\Pi_{u, x}\,.
        \end{aligned}
    \end{equation}
    Hence, we constructed a coupling $\Pi$ such that 
    \[
    \text{GFlow}^\star \ge   \sum_{u \in U}\sum_{x \in X} d(u, x)\Pi_{u, x}\,.
    \]
    For the optimal coupling $\Pi^\star$, therefore, it holds that 
    \begin{equation}
    \label{eq:OT_star_leq}
    \text{GFlow}^\star \geq \text{OT}^{\star}\,.
    \end{equation}
    Combining \eqref{eq:OT_star_geq} and \eqref{eq:OT_star_leq}, we obtain that the GFlow optima is the same as the OT optima, therefore \eqref{eq:ot-graph-problem-value} is proved.
\end{proof}

This theorem give us characterization of a GFlowNet as an implicit way to solve Kantorovich optimal transport, as a corollary we obtain dual equivalence, this essentially follows from the strict duality of linear programming optimization problems. Detailed derivation of the dual to GFlowNet problem in terms of OT functional can be found in Appendix \ref{app:dualgflowot}. 

Furthermore, under appropriate structural assumptions, the optimal transport problem admits a reduction to a shortest-path problem. We show that, in this regime, non-acyclic GFlowNets implicitly recover such shortest-path solutions. The result is stated formally below. 

\begin{theorem}
\label{th:shortestpath}
Consider the dual to the primal problem \eqref{eq:extended-primal-lp} without first-step condition.
\begin{align*}
\max_{\pi} \quad & \sum_{x\in X} \cR(x)\pi_x \\
\text{s.t.} \quad
& \pi_{s_0}=0, \\
& \pi_{s'}-\pi_s \le 1,
\quad \forall (s \to s') \text{ with } s' \neq s_f .
\end{align*}
Then: 1. every feasible solution \(\pi\) satisfies: $\pi_s \le d(s) \quad \forall s\neq s_f$; 2. the function \(d(\cdot) = |\tau_{s_0, \cdot}|\) is a feasible dual solution, and it is dual-optimal; if \(\cR(x)>0\) for every \(x\in X\), then every optimal solution \(\pi^\star\) satisfies: $\pi_x^\star = d(x) \quad \forall x\in X.$

Moreover, complementary slackness conditions in the optimum gives us the following relation:
\[\cF^{\star}(s \to s') (\pi_{s'}^\star - (1 + \pi_s^\star)) = 0\]
which is also a primal-dual gap.

\end{theorem}

The complementary slackness conditions provide a characterization of the optimal flow: it can be non-zero only on the tight subgraph, where $(\pi_{s'}^\star - (1 + \pi_s^\star)) = 0.$ Since $\pi_x^\star$ coincides with the shortest-path metric, it follows that the flow can be supported only on a subgraph of shortest paths. This recovers the corresponding claim of~\citet{morozov2026learningshortestpathsgenerative}. Detailed proof of Theorem \ref{th:shortestpath} can be found in Appendix \ref{app:proof_shortestpath}


\begin{table*}[t]
\begin{center}
    
\begin{footnotesize}
\caption{Comparison of the method on the Hypergrid environment. $\widehat{\mathrm{TV}}$ is the empirical total variation distance between samples from the model and the true distribution, $\mathrm{TV}^{\star}$ is the total variation obtained by a perfect sampler, $\mathbb{E}|\tau|$ is the average trajectory length of the learned sampler, and $\mathrm{OT}^{\star}$ is the optimal OT cost obtained via the POT solver~\cite{flamary2021pot}.}
\label{tab:hypergrid}

\scriptsize
\setlength{\tabcolsep}{2pt}
\renewcommand{\arraystretch}{0.92}

\resizebox{0.67\textwidth}{!}{%
\begin{tabular}{@{}l|cccc|cccc@{}}
\toprule
&
\multicolumn{4}{c|}{$\mathcal{L}=\mathrm{Ball}$}
&
\multicolumn{4}{c}{$\mathcal{L}=\mathrm{Moon}$}
\\
\cmidrule(lr){2-5}
\cmidrule(lr){6-9}
$H$
&
$\widehat{\mathrm{TV}} \downarrow$
&
$\mathrm{TV}^{\star} \downarrow$
&
$\mathbb{E}|\tau|$
&
$\mathrm{OT}^{\star}$
&
$\widehat{\mathrm{TV}} \downarrow$
&
$\mathrm{TV}^{\star} \downarrow$
&
$\mathbb{E}|\tau|$
&
$\mathrm{OT}^{\star}$
\\
\midrule

$10$
& $0.024$ \scalebox{0.7}{$\!\pm\!0.0004$}
& $0.024$
& $3.990$ \scalebox{0.7}{$\!\pm\!0.015$}
& $3.997$
& $0.022$ \scalebox{0.7}{$\!\pm\!0.018$}
& $0.024$
& $4.352$ \scalebox{0.7}{$\!\pm\!0.015$}
& $4.351$
\\

$15$
& $0.036$ \scalebox{0.7}{$\!\pm\!0.0009$}
& $0.033$
& $6.325$ \scalebox{0.7}{$\!\pm\!0.011$}
& $6.303$
& $0.023$ \scalebox{0.7}{$\!\pm\!0.006$}
& $0.033$
& $6.907$ \scalebox{0.7}{$\!\pm\!0.010$}
& $6.868$
\\

$20$
& $0.037$ \scalebox{0.7}{$\!\pm\!0.0006$}
& $0.040$
& $8.326$ \scalebox{0.7}{$\!\pm\!0.021$}
& $8.325$
& $0.032$ \scalebox{0.7}{$\!\pm\!0.008$}
& $0.040$
& $9.001$ \scalebox{0.7}{$\!\pm\!0.112$}
& $9.059$
\\

\bottomrule
\end{tabular}%
}
\end{footnotesize}
\end{center}
\end{table*}

{\renewcommand{\arraystretch}{1.0}
\setlength{\tabcolsep}{3pt}
\begin{table*}[t]
\caption{Results obtained on permutations environment with ablation of $\lambda$. Reference $\text{OT}^\star$ cost obtained via the POT solver~\cite{flamary2021pot} for $n=4$ is $\textcolor{blue}{0.567}$ and is $n=8$ is $\textcolor{red}{1.008}$, for larger permutations $\text{OT}^\star$ is intractable. $C(k)L^1$ is the empirical $L_1$ error between fixed points probabilities obtained from model and the true ones~\cite{pmlr-v267-morozov25a} $\E|\tau|$ is the expected trajectory length. All results are averaged over 3 seeds.}
\label{tab:perms}
\centering
\begin{center}
\begin{footnotesize}
\begin{tabular}{@{}l|cc|cc|cc@{}}
    \toprule
      &
      \multicolumn{2}{c}{$n=4$} & \multicolumn{2}{c}{$n=8$} & \multicolumn{2}{c}{$n=20$}    \\
     \cmidrule(l){2-7}  
     $\lambda$
               & $C(k) \; L^1 \downarrow$& $\E|\tau|$ 
               & $C(k) \; L^1 \downarrow$  & $\E|\tau|$ & $C(k) \; L^1 \downarrow$& $\E|\tau|$ \\
    \midrule
    $\lambda=10^{-1}$
    & $0.012$ \scalebox{0.7}{$\!\pm\!0.001$}
    & $0.445$ \scalebox{0.7}{$\!\pm\!0.002$}
    & $0.011$ \scalebox{0.7}{$\!\pm\!0.005$}
    & $0.645$ \scalebox{0.7}{$\!\pm\!0.013$}
    & $0.016$ \scalebox{0.7}{$\!\pm\!0.000$}
    & $2.313$ \scalebox{0.7}{$\!\pm\!0.018$}\\
    $\lambda=10^{-2}$
    & $0.002$ \scalebox{0.7}{$\!\pm\!0.000$}
    & $\textcolor{blue}{0.557}$ \scalebox{0.7}{$\!\pm\!0.001$}
    & $0.001$ \scalebox{0.7}{$\!\pm\!0.002$}
    & $\textcolor{red}{1.010}$ \scalebox{0.7}{$\!\pm\!0.011$}
    & $0.002$ \scalebox{0.7}{$\!\pm\!0.000$}
    & $4.436$ \scalebox{0.7}{$\!\pm\!0.014$}\\
    \bottomrule
    \end{tabular}
\end{footnotesize}
\end{center}
\label{perms_table}
\end{table*}

\subsection{Learning-based algorithm}
\label{sec:alg}
To approximate the solution to~\eqref{eq:reduced-primal-lp} with neural networks, we utilize non-acylic GFlowNet training methodology. While one can apply various existing GFlowNet losses with flow regularization to train the model~\cite{brunswic2024theory, morozov2024improving}, \citet{morozov2026learningshortestpathsgenerative} recently showed the effectiveness of trajectory balance (TB) objective \cite{malkin2022trajectory} for training non-acyclic GFlowNets in pathfinding problems, which we also found the most suitable in our setup. 

For each state in the graph, the model predicts transition probabilities for forward and backward policies $\PF$ and $\PB$. During training, we sample a batch of trajectories using the forward policy $\PF$, and compute the regularized TB objective:
\begin{equation}
\label{ourloss}
    \begin{aligned}
        \cL_{\mathtt{TB}}(\theta, \tau) &= \bigg( \log \frac{ \cL(s_1) \prod_{t=1}^{n_\tau}  \PF(s_{t+1} | s_{t}, \theta)}{\cR(s_{n_\tau})\prod_{t=0}^{n_\tau - 1}  \PB(s_{t} | s_{t + 1}, \theta)} \bigg)^2\, \\ &+ \lambda \frac{\cR(s_{n_\tau})}{\PF(s_f | s_{n_\tau}, \theta)}.
    \end{aligned}
\end{equation}
The difference from the usual TB objective is the appearance of leward $\cL$ due to the additional conditions on the initial distribution in~\eqref{eq:reduced-primal-lp}. In terms of the policies, reward matching can be written as $\PB(s \mid s_f) = \cR(s)$, and leward matching as $\PF(s \mid s_0) = \cL(s)$, which explains the form of the loss in~\eqref{ourloss}. The regularizer $\lambda {\cR(s)} / {\PF(s_f | s)}$ corresponds to the state flow $\cF(s)$ in terminal states due to the reward matching conditions in \eqref{eq:gflownet-min-flow}. \citet{morozov2026learningshortestpathsgenerative} also showed that if each internal state in $\cG$ is terminal, i.e. $\cX = I$, the objective~\eqref{ourloss} can be computed for every prefix of the sampled trajectory, resulting in a more sample-efficient approach. We utilize this option in our experiments.





\section{Experiments}
\label{allexp}


In this section, we experimentally evaluate the proposed GFlowNet-based method for solving graph OT tasks. We consider the tasks based on the environments studied in~\citet{brunswic2024theory, pmlr-v267-morozov25a}. On a smaller hypergrid task, we empirically demonstrate that the exact solution to the GFlowNet LP formulation \eqref{eq:reduced-primal-lp} is equivalent to Kantorovich OT formulation \eqref{eq:ot-graph-problem}, and can be obtained both by the exact LP solver and learning-based neural network approximation described in Section~\ref{sec:alg}. We then validate our learning-based method on larger-scale permutation environments. In particular, we recover the optimal OT solution for permutations of length $n=4$ and $n=8$, and demonstrate convergence for $n=20$. Finally, we ablate the effect of the state-flow regularization parameter $\lambda$, illustrating its impact on the optima reached by the method.



\subsection{Hypergrids}
\label{exp:hypergrids}

We evaluate our method on OT tasks on non-acyclic hypergrid environments proposed in~\citet{brunswic2024theory}. Since these environments remain small enough for exact LP optimization, they allow direct comparison between the learned sampler and the optimal OT solution. For additional details see Appendix~\ref{app:grid_app}.


States are lattice points in $\{0,\ldots,H-1\}^D$ augmented with source and terminal states $s_0,s_f$. Sampling starts from $U$ with distribution $\cL$ and transitions modify a single coordinate by $\pm1$ while remaining inside the grid. Every state also admits a terminating transition to $s_f$. We assume that leward $\cL$ admits tractable sampling, and focus on measuring two quantities: 1) what is the average transport cost of the trained model and 2) how close is the distribution of terminal states sampled from the model to the reward distribution $\cR$.

We consider multimodal rewards with modes near the grid corners~\cite{bengio2021flow} and use moon-shaped and ball-shaped leward distributions. Sampling quality is measured via the total variation distance between the reward distribution and the empirical distribution of $2\cdot10^5$ model samples. Table~\ref{tab:hypergrid} reports the resulting $\widehat{\mathrm{TV}}$, compares it with the perfect-sampler reference $\mathrm{TV}^\star$ (it is non-zero because and empirical distribution is considered), and evaluates convergence to the OT optimum $\mathrm{OT}^\star$ using the mean trajectory length. One can clearly see that our learned policy recovers optimal $\text{OT}^{\star}$ without biasing the sampling across all selected leward distributions and all hypergrid sides. All results are averaged over three seeds.

Figure~\ref{fig:exact_solutions} additionally visualizes the OT solution for $H = 10$ and $\cL = \text{Moon}$, as well as the exact solution of the equivalent GFlowNet LP formulation~\eqref{eq:reduced-primal-lp} obtained via scipy.linprog solver \cite{2020SciPy-NMeth}.

\subsection{Permutations}
Next, we consider the environment induced by the Cayley graph of the symmetric group $\mathrm{S}_n$, the group of permutations of $n$ elements ${1,2,\dots,n}$, proposed in~\citet{pmlr-v267-morozov25a}. Each internal state $s \in \cS \setminus {s_0,s_f}$ corresponds to a permutation of fixed length, $(s(1), \dots, s(n))$, Transitions are defined by adjacent transpositions: from a given state, one may swap two neighboring entries $s(k)$ and $s(k+1)$.  We take $\cL$ to be the uniform distribution over all permutations, and use the reward distribution as $\cR(s) = \exp\left(\frac{1}{2}\sum_{k=1}^n\mathbb{I}\{s(k) = k\}\right)/\cZ$ (\citet{pmlr-v267-morozov25a} derived closed form expression for its normalizing constant). We adopt the diagnostic protocol of~\citet{pmlr-v267-morozov25a}; a detailed description of the corresponding metric is provided in Appendix~\ref{app:perms_metrics}. Table~\ref{tab:perms} reports results on permutation environments, measuring both the reward sampling quality and transport cost similarly to Section~\ref{exp:hypergrids}. All results are averaged over three random seeds.

For small permutation sizes, $n=4$ and $n=8$, our method recovers solutions very close to optimal, achieving an average trajectory length almost equal to $\mathrm{OT}^\star$. For larger permutations, such as $n=20$, computing the exact optimal OT solution is no longer tractable; nevertheless, our method produces a reasonable approximation. We further ablate the effect of the state-flow regularization coefficient $\lambda$. Consistently with~\citet{pmlr-v267-morozov25a}, we observe that $\lambda$ controls a trade-off between the trajectory length of the learned sampler and sampling optimality. The larger $\lambda$ brings bias to the sampling but reduces the trajectory length, on the contrary, the lower value induces accurate sampling but larger trajectories.

\vspace{-0.15cm}
\section{Conclusion}
\label{sec:conclusion}
\vspace{-0.1cm}
We extended the theoretical framework of non-acyclic GFlowNets by establishing an equivalence with optimal transport on graphs. Unlike standard OT approaches that only learn a coupling, our framework learns a stochastic policy that transports samples through feasible local moves with near-optimal expected path cost, thereby addressing both where and how transport occurs.
We further provided a novel interpretation of the shortest-path formulation of GFlowNets from~\citet{morozov2026learningshortestpathsgenerative}. The equivalence established in Theorem~\ref{th:otequivalence} opens several promising directions, including backward policy optimization algorithms~\cite{jang2024pessimistic, gritsaev2024optimizing} and reinforcement learning techniques motivated by the connection of GFlowNets to entropy-regularized RL~\cite{tiapkin2024generative, deleu2024discrete, mohammadpour2024maximum, lau2024qgfn, morozov2024improving, he2025random}. Future work could also address practical challenges highlighted by our experimental evaluation. While presenting a computationally efficient training method, non-acyclic GFlowNets include the regularization coefficient $\lambda$ that must be carefully tuned to balance OT optimality and sampling quality. In addition, existing GFlowNet methodology could be potentially applied to learn transport maps between unnormalized probability distributions.


\section*{Acknowledgements}

This research was supported in part through computational resources of HPC facilities at HSE University~\citep{kostenetskiy2021hpc}. We would like to thank FORTUNA 812 for his inspirational music that kept us sane during this research.



\bibliography{bibliography}
\bibliographystyle{icml2026}

\newpage
\appendix
\onecolumn
\newpage 

\section{Full theoretical results}
\label{app:all_proofs}
\subsection{Correctness of LP formulation}
\label{app:correctens_lp}
Consider original minimum-flow formulation \eqref{eq:gflownet-min-flow}, we will now proof that it could be indeed reformulated as linear programming problem.

\begin{equation}
\begin{aligned}
\min_{\cF,\PF,\PB}\quad
& \sum_{s\in I} \cF(s) \\
\text{s.t.}\quad
& \cF(s)\PF(s'\mid s)=\cF(s')\PB(s\mid s'),\\
&\qquad s\to s'\in \cE, \\
& \cF(s_f)\PB(x\mid s_f)=\cR(x),
&& x\in X.
\end{aligned}
\end{equation}

\begin{proposition}[Elliminating forward policy]
    Fix \(s \neq s_f\). The following are equivalent:
\begin{enumerate}
    \item There exists a probability distribution \(\PF(\cdot \mid s) \in \Delta(\text{out}(s))\) such that
    \[
    \cF(s \to v) = \cF(s) \PF(v \mid s),
    \qquad \forall v \in \text{out}(s).
    \]
    \item The variables \(\cF(s \to v)\) satisfy the linear conditions
    \[
    \cF(s \to v) \ge 0 \quad \forall v \in \text{out}(s),
    \qquad
    \sum_{v \in \text{out}(s)} \cF(s \to v) = \cF(s).
    \]
\end{enumerate}
\end{proposition}

\begin{proof}
Assume \textup{(i)}. Since \(\PF(\cdot \mid s)\) is a probability distribution,
\[
\sum_{v \in \text{out}(s)} \cF(s \to v)
=
\sum_{v \in \text{out}(s)} \cF(s) \PF(v \mid s)
=
\cF(s) \sum_{v \in \text{out}(s)} \PF(v \mid s)
=
\cF(s),
\]
and clearly \(\cF(s \to v) \ge 0\) for all \(v\).

Now assume \textup{(ii)}. If \(\cF(s) > 0\), define
\[
\PF(v \mid s) := \frac{\cF(s \to v)}{\cF(s)},
\qquad v \in \text{out}(s).
\]
Then \(\PF(v \mid s) \ge 0\) and
\[
\sum_{v \in \text{out}(s)} \PF(v \mid s) = \frac{1}{\cF(s)} \sum_{v \in \text{out}(s)} \cF(s \to v) = 1.
\]
Hence \(\PF(\cdot \mid s)\in \Delta(\text{out}(s)\), and by construction
\[
\cF(s \to v) = \cF(s) \PF(v \mid s).
\]
If \(\cF(s) = 0\), then \(\sum_{v} \cF(s \to v)=0\) and \(\cF(s \to v)\ge0\) imply \(\cF(s \to v)=0\) for all \(v\), so any choice of \(\PF(\cdot \mid s)\in \Delta(\text{out}(s))\) works.
\end{proof}

\begin{proposition}[Eliminating the backward policy]
Fix \(s' \neq s_0\). The following are equivalent:
\begin{enumerate}
    \item There exists a probability distribution \(\PB(\cdot \mid s') \in \Delta(\text{in}(s'))\) such that
    \[
    \cF(u \to s') = \cF(s') \PB(u \mid s'),
    \qquad \forall u \in \text{in}(s').
    \]
    \item The variables \(\cF(u \to s')\) satisfy the linear conditions
    \[
    \cF(u \to s') \ge 0 \quad \forall u \in \text{in}(s'),
    \qquad
    \sum_{u \in \text{in}(s')} \cF(u \to s') = \cF(s').
    \]
\end{enumerate}
\end{proposition}

\begin{proof}
The proof is identical to the previous one, with incoming edges in place of outgoing edges.
\end{proof}

\begin{corollary}[Lifted LP]
\label{eq:lifted_lp_problem}
After eliminating \(\PF\) and \(\PB\), the minimum-total-flow problem becomes
\[
\begin{aligned}
\textup{(P)}\qquad
\min_{\cF(s), \cF(s \to s')}\quad & \sum_{s\in I} \cF(s)\\
\text{s.t.}\quad
& \sum_{v\in \text{out}(s)} \cF(s \to v) = \cF(s),
&& \forall s \in S\setminus\{s_f\},\\
& \sum_{u\in \text{in}(s)} \cF(u \to s) = \cF(s),
&& \forall s \in S\setminus\{s_0\},\\
& \cF(x \to s_f) = \cR(x),
&& \forall x \in X,\\
& \cF(s \to s') \ge 0,
&& \forall (s,s') \in E.
\end{aligned}
\]
This is a linear program.
\end{corollary}

\begin{proof}
The objective is linear in \((\cF(s),\cF(s \to s'))\), and all constraints are affine.
\end{proof}

\subsection{Duality proofs}
\label{sec:duality_proofs}
For notational convenience define the cost vector \(c\) by
\[
c_s :=
\begin{cases}
1, & s\in I,\\
0, & s\in \{s_0,s_f\}.
\end{cases}
\]
Then the primal objective is
\[
c^\top \cF = \sum_{s\in I} \cF(s)\,.
\]

Introduce dual variables:
\begin{itemize}
    \item \(\alpha_s \in \) for the outgoing-flow constraints, \(s\in S\setminus\{s_f\}\);
    \item \(\beta_s \in \) for the incoming-flow constraints, \(s\in S\setminus\{s_0\}\);
    \item \(\eta_x \in \) for the terminal constraints, \(x\in X\).
\end{itemize}

We also define padded vectors \(\bar{\alpha},\bar{\beta}\) by
\[
\bar{\alpha}_s :=
\begin{cases}
\alpha_s, & s\neq s_f,\\
0, & s=s_f,
\end{cases}
\qquad
\bar{\beta}_s :=
\begin{cases}
\beta_s, & s\neq s_0,\\
0, & s=s_0.
\end{cases}
\]

\begin{proposition}
\label{app:dual_derivation}
The dual of \eqref{eq:lifted_lp_problem} is
\begin{equation}
    \begin{aligned}
\label{app:proof_dual}
\qquad
\max_{\alpha,\beta,\eta}\quad
& \sum_{x\in X} \cR(x)\eta_x\\
\text{s.t.}\quad
& c = \bar{\alpha} + \bar{\beta},\\
& \alpha_s + \beta_{s'} \ge 0,
&& \forall s \to s' \text{ with } s' \neq s_f,\\
& \eta_x \le \alpha_x + \beta_{s_f},
&& \forall x\in X.
\end{aligned}
\end{equation}

Moreover, this dual can be reduced: 

\begin{equation}
\begin{aligned}
\label{app:proof_dual_reduced}
\max_{\pi}\quad
& \sum_{x\in X} \cR(x)\pi_x\\
\text{s.t.}\quad
& \pi_{s_0}=0,\\
& \pi_{s'}-\pi_s \le 1,
&& \forall (s,s')\in E \text{ with } s' \neq s_f.
\end{aligned}    
\end{equation}

\end{proposition}

\begin{proof}
    The Lagrangian is
\[
\begin{aligned}
L(\cF(s),\cF(s \to s'),\alpha,\beta,\eta)
&=
c^\top \cF
+ \sum_{s\neq s_f}
\alpha_s \Bigl(\sum_{v\in\text{out}(s)} \cF(s \to v) - \cF(s)\Bigr)\\
&\qquad
+ \sum_{s\neq s_0}
\beta_s \Bigl(\sum_{u\in\text{in}(s)} \cF(u \to s) - \cF(s)\Bigr)
+ \sum_{x\in X} \eta_x \bigl(\cR(x) - \cF(x \to s_f)\bigr).
\end{aligned}
\]
Now collect all terms involving \(\cF(s)\) and \(\cF(s \to s')\). Since each edge \(s\to s'\) contributes
\(\alpha_s\) through the outgoing constraint at \(s\), and \(\beta_{s'}\) through the incoming
constraint at \(s'\), we obtain
\[
\begin{aligned}
L(\cF(s),\cF(s \to s'),\alpha,\beta,\eta)
&=
\sum_{x\in X} \cR(x)\eta_x
+ \sum_{s\in S} \bigl(c_s - \bar{\alpha}_s - \bar{\beta}_s\bigr) \cF(s)\\
&\qquad
+ \sum_{\substack{(s,s')\in E\\ s'\neq s_f}}
(\alpha_s+\beta_{s'})\, \cF(s \to s')\\
&\qquad
+ \sum_{x\in X}
(\alpha_x+\beta_{s_f}-\eta_x)\, \cF(x \to s_f).
\end{aligned}
\]
The dual function is
\[
g(\alpha,\beta,\eta) := \inf_{\cF(s),\; \cF(s \to s')\ge0} L(\cF(s),\cF(s \to s'),\alpha,\beta,\eta).
\]

We now determine when this infimum is finite.

\medskip
\noindent
\textbf{Step 1: minimization over \(\cF(s)\).}
The variables \(\cF(s)\) are unrestricted in sign inside the Lagrangian. Therefore the infimum over \(\cF(s)\)
is finite if and only if each coefficient of \(\cF(s)\) is zero:
\[
c_s - \bar{\alpha}_s - \bar{\beta}_s = 0,
\qquad \forall s\in S.
\]
Equivalently,
\[
c = \bar{\alpha} + \bar{\beta}.
\]

\medskip
\noindent
\textbf{Step 2: minimization over \(\cF(s \to s')\).}
Each variable \(\cF(s \to s')\) is constrained only by \(\cF(s \to s') \ge 0\).
Hence the infimum over \(\cF(s \to s')\) is finite if and only if each coefficient of \(\cF(s \to s')\) is nonnegative:
\[
\alpha_s + \beta_{s'} \ge 0,
\qquad \forall (s \to s') \text{ with } s'\neq s_f,
\]
and
\[
\alpha_x + \beta_{s_f} - \eta_x \ge 0,
\qquad \forall x\in X.
\]
The latter condition is equivalent to
\[
\eta_x \le \alpha_x + \beta_{s_f},
\qquad \forall x\in X.
\]

Under these conditions,
\[
g(\alpha,\beta,\eta) = \sum_{x\in X} R(x)\eta_x.
\]
Therefore the dual problem is exactly \eqref{app:proof_dual}.

Since \(R(x)>0\) for every \(x\in X\), the optimal choice of \(\eta_x\) saturates the inequality
\[
\eta_x \le \alpha_x + \beta_{s_f},
\]
hence
\[
\eta_x = \alpha_x + \beta_{s_f}.
\]
Therefore \eqref{app:proof_dual} is equivalent to
\[
\begin{aligned}
\max_{\alpha,\beta}\quad
& \sum_{x\in X} R(x)\bigl(\alpha_x+\beta_{s_f}\bigr)\\
\text{s.t.}\quad
& c = \bar{\alpha} + \bar{\beta},\\
& \alpha_s + \beta_{s'} \ge 0,
&& \forall (s \to s') \text{ with } s' \neq s_f.
\end{aligned}
\]

From \(c=\bar{\alpha}+\bar{\beta}\), we obtain
\[
\alpha_{s_0}=0,\qquad \beta_{s_f}=0,\qquad \beta_s = 1-\alpha_s \quad \forall s\in I.
\]
Substituting into the edge inequalities gives
\[
\alpha_s+\beta_{s'} \ge 0
\quad \Longleftrightarrow \quad
\alpha_s + (1-\alpha_{s'}) \ge 0
\quad \Longleftrightarrow \quad
\alpha_{s'}-\alpha_s \le 1.
\]
Renaming \(\alpha_s\) as \(\pi_s\) yields \eqref{app:proof_dual_reduced}.

Note that the derivation of dual to the extended problem \ref{eq:extended-primal-lp} is exactly the same up to an additional equality constraint from the fixed first step.
\end{proof}

\subsection{Shortest Path via Dual Potentials}
\label{app:shortest_path}
\begin{theorem}[Shortest paths dual characterization]
\label{app:proof_shortestpath}
Consider the dual to the primal problem \ref{app:dual_derivation}
\[
\max_{\pi}\ \sum_{x\in X} \cR(x)\pi_x
\qquad\text{subject to}\qquad
\pi_{s_0}=0,
\qquad
\pi_{s'}-\pi_s\le 1
\quad \forall (s \to s') \text{ with } s'\neq s_f.
\]
For every non-sink state \(s\), define
\[
d(s):=\min\left\{k\ge 0:\ \exists\ s_0=v_0\to v_1\to \cdots \to v_k=s\right\}.
\]
Then:
\begin{enumerate}
    \item every feasible solution \(\pi\) satisfies
    \[
    \pi_s \le d(s)
    \qquad \forall s\neq s_f;
    \]
    \item the function \(d(\cdot)\) is a feasible dual solution, and it is dual-optimal;
    \item if \(\cR(x)>0\) for every \(x\in X\), then every optimal solution \(\pi^\star\) satisfies
    \[
    \pi_x^\star = d(x)
    \qquad \forall x\in X.
    \]
\end{enumerate}

Moreover, complementary slackness conditions in the optimum gives us the following relation:
\[\cF^{\star}(s \to s') (\pi_{s'}^\star - (1 + \pi_s^\star)) = 0\]
which is also a Primal-Dual gap.

\end{theorem}

\begin{proof}
We prove the three statements one by one.

\medskip
\noindent
\textbf{Step 1: every feasible \(\pi\) satisfies \(\pi_s\le d(s)\).}

Let \(\pi\) be any feasible solution of the dual problem, and fix a state \(s\neq s_f\).
Take any directed path from \(s_0\) to \(s\):
\[
s_0=v_0 \to v_1 \to \cdots \to v_k=s.
\]
For every edge \(v_i\to v_{i+1}\) in this path, feasibility gives
\[
\pi_{v_{i+1}}-\pi_{v_i}\le 1.
\]
Now sum these inequalities over \(i=0,\dots,k-1\):
\[
\sum_{i=0}^{k-1}\bigl(\pi_{v_{i+1}}-\pi_{v_i}\bigr)\le \sum_{i=0}^{k-1}1 = k.
\]
The left-hand side is a telescoping sum, so it simplifies to
\[
\pi_{v_k}-\pi_{v_0} = \pi_s-\pi_{s_0}.
\]
Since \(\pi_{s_0}=0\), we obtain
\[
\pi_s\le k.
\]
This holds for every path from \(s_0\) to \(s\). Therefore, taking the minimum over all such paths,
\[
\pi_s\le d(s).
\]
This proves part \textup{(i)}.

\medskip
\noindent
\textbf{Step 2: the function \(d(\cdot)\) is feasible.}

Define
\[
\pi^{\max}_s := d(s).
\]
We show that \(\pi^{\max}\) satisfies all dual constraints.

First,
\[
\pi^{\max}_{s_0}=d(s_0)=0.
\]

Next, take any edge \(s\to s'\) with \(s'\neq s_f\).
By definition of \(d(s)\), there exists a shortest path from \(s_0\) to \(s\) of length \(d(s)\):
\[
s_0=v_0\to v_1\to\cdots\to v_{d(s)}=s.
\]
Appending the edge \(s\to s'\) gives a path from \(s_0\) to \(s'\) of length \(d(s)+1\).
Since \(d(s')\) is the length of the shortest path from \(s_0\) to \(s'\), we must have
\[
d(s')\le d(s)+1.
\]
Equivalently,
\[
d(s')-d(s)\le 1.
\]
Thus
\[
\pi^{\max}_{s'}-\pi^{\max}_s \le 1.
\]
So \(\pi^{\max}\) is feasible.

\medskip
\noindent
\textbf{Step 3: \(d(\cdot)\) is dual-optimal.}

Let \(\pi\) be any feasible solution. By Step 1, for every terminal state \(x\in X\),
\[
\pi_x\le d(x).
\]
Since the rewards are nonnegative,
\[
\sum_{x\in X} \cR(x)\pi_x \le \sum_{x\in X} \cR(x)d(x).
\]
But Step 2 shows that the choice \(\pi^{\max}_s=d(s)\) is feasible, and for this choice the objective value is exactly
\[
\sum_{x\in X} \cR(x)d(x).
\]
Therefore no feasible solution can have a larger objective value, and \(d(\cdot)\) is dual-optimal.
This proves part \textup{(ii)}.

\medskip
\noindent
\textbf{Step 4: Every optimal solution agrees with \(d\) on terminals.}

Let \(\pi^\star\) be any optimal solution.
By Step 1,
\[
\pi^\star_x\le d(x)
\qquad \forall x\in X.
\]
Suppose, for contradiction, that there exists some terminal state \(x_0\in X\) such that
\[
\pi^\star_{x_0}<d(x_0).
\]
Because all rewards are strictly positive, this would imply
\[
\sum_{x\in X}\cR(x)\pi^\star_x
<
\sum_{x\in X}\cR(x)d(x).
\]
But the right-hand side is the optimal value, since it is attained by the feasible solution \(d(\cdot)\).
This contradicts the optimality of \(\pi^\star\).

Hence no such \(x_0\) can exist, and therefore
\[
\pi^\star_x=d(x)
\qquad \forall x\in X.
\]
This proves part \textup{(iii)}.
    
\end{proof}


\subsection{Proof of dual equivalence of OT and GFlowNet}
\label{app:dualgflowot}
\begin{proof}
It is easy to see that dual to the extended primal problem is of the following nature (see Appendix \ref{app:dual_derivation} for derivation)
\begin{equation}
\begin{aligned}
    \max_{\pi} \quad 
    &\sum_{x\in X} \cR(x) \pi_x + \sum_{u \in U} \cL(u)(-\pi_u) \\
    \text{s.t.}\quad
    & \pi_{s'} - \pi_s \le 1, \forall s\rightarrow s': s \neq s_0, s' \neq s_f \\
    & \pi_{s_0} = 0 
\end{aligned} 
\end{equation}
if we now redefine $a_x = \pi_x, b_u = -\pi_u$, then taking any path $p$:
\[a_x + b_u = \pi_x -\pi_u=\sum_{i=0}^{|p|} \pi_{p_{i + 1}} - \pi_{p_i} \leq d(u, x)\]
Hence:
\begin{equation}
\begin{aligned}
    \max_{a, b} \quad 
    &\sum_{x\in X} \cR(x) a_x + \sum_{u \in U} \cL(u)b_u \\
    \text{s.t.}\quad
    & a_x + b_u \le d(u, x)\\
\end{aligned} 
\end{equation}
Which has exactly the same functional form in decision variables $a_x, b_u$ as Dual to Problem \ref{eq:ot-kantorovich}
    
\end{proof}

\section{Experimental details}
\label{app:exprimental_details}
All models are parameterized by MLP with 2 hidden layers and 128 hidden size, which accept a one-hot encoding of $s$ as input. $\cF_{\theta}(s), \PF(s' \mid s, \theta), \PB(s \mid s', \theta)$ share the same backbone, with different linear heads predicting  the logits of the forward policy and the logits of the backward policy. We train our models on-policy, using batch size $512$.  We use AdamW optimzier with $\text{lr} =10^{-3}$ for all our experiments and weight decay $10^{-4}$. Our implementations are based on the published code of \cite{pmlr-v267-morozov25a}. All experiments were performed on CPUs.

\subsection{Hypergrid distribution definitions}
\label{app:grid_app}
We define moon shaped distribution by the following formula: 
\[
\cL(s) \triangleq \ind\left[\|z-c\|_2^2 \le r_{\mathrm{out}}^2\right]\left(1-\ind\left[\|z-(c-\delta e_1)\|_2^2 \le r_{\mathrm{in}}^2\right]\right)\left(0.5+2\min\left\{1,\max\left\{0,1-\frac{\|z-b\|_2}{r_{\mathrm{out}}}\right\}\right\}\right)+\varepsilon,
\]

Where:
\[
z=\frac{s}{M-1},
\qquad
c=\left(\frac12,\dots,\frac12\right),
\qquad
e_1=(1,0,\dots,0),
\qquad
b=c+\frac{r_{\mathrm{out}}}{2}e_1.
\]
and $r_{\text{in}}, r_{\text{out}}$ are inner and outer balls radius.

The ball-shaped distribution defined as follows: 

\[
\cL(s) \triangleq \ind\left[\|z-c\|_2^2 \le r_{\mathrm{out}}^2\right]\left(0.5+2\min\left\{1,\max\left\{0,1-\frac{\|z-b\|_2}{r_{\mathrm{out}}}\right\}\right\}\right)+\varepsilon,
\]

The corner-shaped distribution definition follows previous works ~\cite{madan2023learning, malkin2022trajectory}: 
\[
\cR(s) \triangleq R_0 + R_1 \prod_{i = 1}^D \mathbb{I}\left[0.25 < \left|\frac{s^i}{H-1}-0.5\right|\right] + R_2 \prod_{i = 1}^D \mathbb{I}\left[0.3 < \left|\frac{s^i}{H-1}-0.5\right| < 0.4\right]\eqsp,
\]

\begin{figure}
    \centering
    \includegraphics[width=1\linewidth]{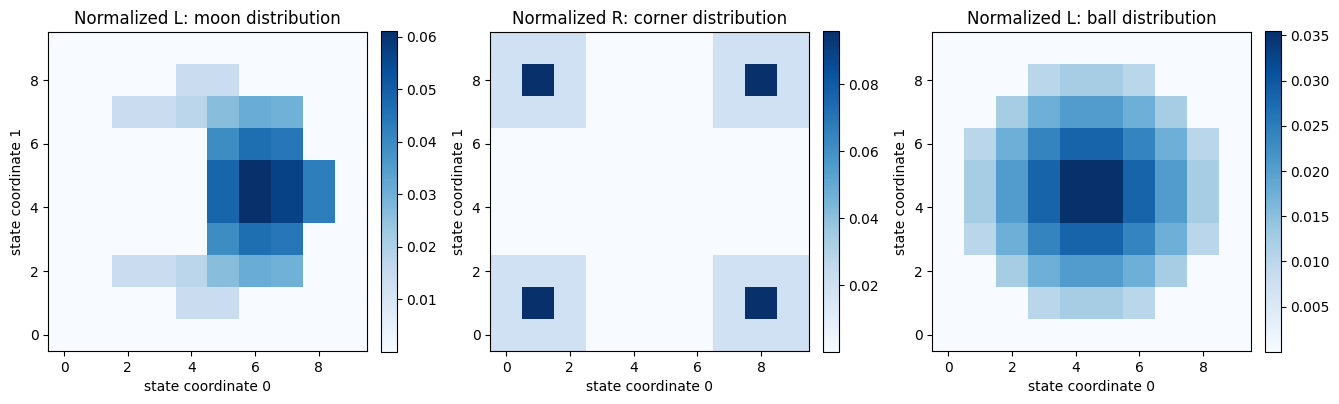}
    \caption{Visualization of the distributions used in the experiments: moon-shaped (left), corner-shaped (center), and ball-shaped (right).\vspace{-0.2cm}}
    \label{fig:rewards}
\end{figure}

See Figure \ref{fig:rewards} for visualization.

\subsection{Permutations}
\label{app:perms_metrics}
Denote $C(k)$ as the probability that a permutation sampled from the reward distribution has $k$ fixed points. We compute the $L^1$ error between the vector $(C(0), C(1), \dots, C(n))$ and its empirical estimate over the last $10^5$ samples seen in training for convergence diagnostics.

Suppose that $x_1, \dots, x_m$ is a set of GFlowNet samples (terminal states of trajectories sampled from $\PF$). Then, the empirical $L_1$ error of fixed point probabilities is defined as:
$$
\sum_{k=0}^N \left| C(k) - \frac{1}{m}\sum_{i=1}^m \mathbb{I}\{x_i(k)=k\}\right|,
$$
We use the exact formula from \cite{pmlr-v267-morozov25a} for exact computation of $C(k)$.
\end{document}